\documentclass{article}

\PassOptionsToPackage{numbers, compress}{natbib}

\usepackage[main, preprint]{neurips_2026}

\usepackage[utf8]{inputenc} 
\usepackage[T1]{fontenc}    
\usepackage{hyperref}       
\usepackage{url}            
\usepackage{booktabs}       
\usepackage{amsfonts}       
\usepackage{nicefrac}       
\usepackage{microtype}      
\usepackage{xcolor}         
\usepackage{array}
\usepackage[normalem]{ulem}
\usepackage{amsmath,amssymb,amsthm,mathtools,bm}

\RequirePackage{algorithm}
\RequirePackage{algorithmic}

\newtheorem{theorem}{Theorem}
\newtheorem{proposition}{Proposition}
\newtheorem{corollary}{Corollary}

\newcommand{\R}{\mathbb{R}}
\newcommand{\Mcal}{\mathcal{M}}
\newcommand{\T}{\mathsf{T}}
\newcommand{\ip}[2]{\left\langle #1, #2 \right\rangle}
\newcommand{\norm}[1]{\left\lVert #1 \right\rVert}

\newcommand{\rank}{\operatorname{rank}}
\newcommand{\msign}{\operatorname{msign}}
\newcommand{\sym}{\operatorname{sym}}

\newcommand{\argmin}{\operatorname*{arg\,min}}
\newcommand{\argmax}{\operatorname*{arg\,max}}
\newcommand{\GL}{\operatorname{GL}}
\newcommand{\secref}[1]{\hyperref[#1]{Section~\ref*{#1}}}
\newcommand{\appendixref}[1]{\hyperref[#1]{Appendix~\ref*{#1}}}

\newcommand{\aireviewer}[1]{}

\title{LoRA-Muon: Spectral Steepest Descent on the Low-Rank Manifold}
\workshoptitle{LoRA-Muon}

\author{%
  Franz Louis Cesista \\
  Ateneo de Manila University; EleutherAI \\
  \texttt{franzlouiscesista@gmail.com} \\
  \And
  Katherine Crowson \\
  EleutherAI \\
  \texttt{crowsonkb@gmail.com} \\
  \And
  Cédric Simal \\
  NaXys, UNamur; EleutherAI \\
  \texttt{cedric.simal@unamur.be}
  \And
  Stella Biderman \\
  EleutherAI \\
  \texttt{stella@eleuther.ai}
}

\begin{document}

\graphicspath{{figures/}}

\maketitle

\begin{abstract}
Low-Rank Adaptation (LoRA) significantly reduces compute and memory costs for finetuning Deep Learning models but is often harder to tune than dense training: when using factor-wise optimizers such as AdamW, it is sensitive to initialization choices, its optimal learning rates transfer poorly across ranks, and it often fails to beat dense baselines. We derive LoRA-Muon by applying the Muon optimizer's spectral steepest-descent rule to the low-rank setting. Along with our split weight-decay rule, our main claim is that LoRA-Muon is a good low-rank proxy for full-rank Muon and Shampoo-family optimizers. Its optimal learning rates transfer across rank, width, depth, and factor-rescaling.
In our compute-matched TinyShakespeare study, a rank-$2$ proxy recovers the dense best tested learning rate, and a rank-$32$ LoRA-Muon run attains lower mean validation loss than the dense baseline in the seed-averaged sweep. We further show that the Spectron optimizer depends on arbitrary factor scaling, so it would likely be a poor fit when finetuning starts from badly imbalanced factors, and that LoRA-RITE's simplified QR-coordinate core implements the same spectral update. LoRA-Muon computes that update without QR-decomposition and avoids storing second moments, making it more accelerator-friendly and memory-efficient.
\end{abstract}

\section{Introduction}\label{sec:intro}

Low-Rank Adaptation (LoRA) is widely used in production settings in the industry because it cuts both training FLOPs and accelerator memory requirements \cite{hu2022lora}.
But standard optimizers such as AdamW \cite{kingma2015adam,loshchilov2019adamw}, when applied independently to the LoRA factors, make LoRA harder to tune than dense training.
And once the factor scales drift apart, training dynamics often diverge even when starting from the exact same $W = AB^T$ composed weights.
Recent work also show that hyperparameters such as the learning rate depend strongly on rank, initialization, and scaling conventions \cite{chen2026mua}, and native low-rank pretraining often needs extra stabilization to compete with dense baselines \cite{janson2026spectron}.
We therefore want an optimizer whose optimal learning rates transfer across rank, width, depth, and factor scales while maintaining strong performance even at extremely low rank.

The crux of \emph{Old Optimizer, New Norm} is that optimizers become clearer once we ask a simpler question: \emph{what norm is this optimizer really doing steepest descent in?} \cite{bernstein2024oldnorm}.
Under that lens, the Muon optimizer performs steepest descent under the spectral norm in $\mathcal{M} = \mathbb{R}^{m \times n}$ \cite{jordan2024muon,bernstein2025derivingmuon,liu2025muonscalable}. That viewpoint turns optimizer design into a (non-Euclidean) geometric problem.

In this paper, we derive a novel family of LoRA optimizers by directly solving the steepest descent problem in the low-rank manifold, $\mathcal{M}_{r} = \{ W \in \mathbb{R}^{m \times n} : \text{rank}(W) = r \}$, instead. Specializing to the spectral norm then yields LoRA-Muon whose optimal learning rates transfer across rank, depth, width, and factor-rescaling, making it a good low-rank proxy of the Muon optimizer (and the Shampoo-family of optimizers as they produce updates of matching spectral norm as Muon \cite{gupta2018shampoo, anil2021scalablesecondorderoptimization, cesista2025casprmuon}). And with modern hyperparameter scaling laws, knowledge of the optimal learning rate is often sufficient
to derive optimal values for other hyperparameters such as weight decay, batch size, training horizon,
and momentum as we scale other parameters \cite{kosson2026weightdecay,shulgin2026hyperscaling,islamov2026batchsize}.

In our compute-matched TinyShakespeare experiments, a rank-$2$ LoRA-Muon proxy already recovers the best tested learning rate for dense Muon. The same best tested learning rate indeed transfers across the rank, width, depth, and factor-rescaling sweeps. At rank $32$, LoRA-Muon attains lower mean validation loss than the dense baseline in the seed-averaged TinyShakespeare sweep. We also show that Spectron is sensitive to arbitrary factor scaling, while LoRA-RITE's simplified QR-coordinate core turns out to be the same spectral update written in QR coordinates.

The rest of this paper is structured as follows. \secref{sec:preliminaries} collects the preliminaries: steepest descent under a norm, the Muon
optimizer, and the low-rank geometry of LoRA. \secref{sec:derive-lora-muon} derives LoRA-Muon.
\secref{sec:gauge} formalizes the symmetry result. \secref{sec:spectron-rite} compares against
Spectron and LoRA-RITE. \secref{sec:experiments} gives the empirical study, and
\secref{sec:limitations} closes with the main limitations.

\subsection{Our contributions}\label{subsec:contributions}
Our contributions are:
\begin{enumerate}
  \item We derive a family of Linear Minimization Oracles (LMOs) for steepest descent in the low-rank manifold under unitary-invariant norms. Specializing to the spectral norm yields LoRA-Muon, a low-rank proxy for the Muon optimizer.
  \item We derive a split weight-decay rule that ensures weight norms and step sizes in the full-rank and low-rank settings match.
  \item We prove that the LoRA-Muon update on $W$ is gauge-invariant under the full
  action $(A, B) \mapsto (AR, BR^{-\T})$, for arbitrary invertible $R$.
  \item We prove that Spectron is not gauge-invariant, already failing under scalar gauge rescaling.
  \item We show that LoRA-RITE's simplified spectral QR-coordinate core is algebraically identical to LoRA-Muon, so both arise from the same steepest-descent construction. LoRA-Muon, however, computes the update without QR factorizations or second moments, thereby greatly improving accelerator-friendliness and memory-efficiency.
  \item We empirically show in Large Language Model training sweeps that LoRA-Muon is a useful proxy for dense Muon learning-rate search: the dense best tested learning rate transfers across rank, width, depth, and factor rescaling, already at rank $2$, and (compute-matched) rank-$32$ LoRA-Muon attains lower mean validation loss than the dense baseline in the seed-averaged sweep.
\end{enumerate}

\begin{figure}
    \centering
    \includegraphics[width=0.7\linewidth]{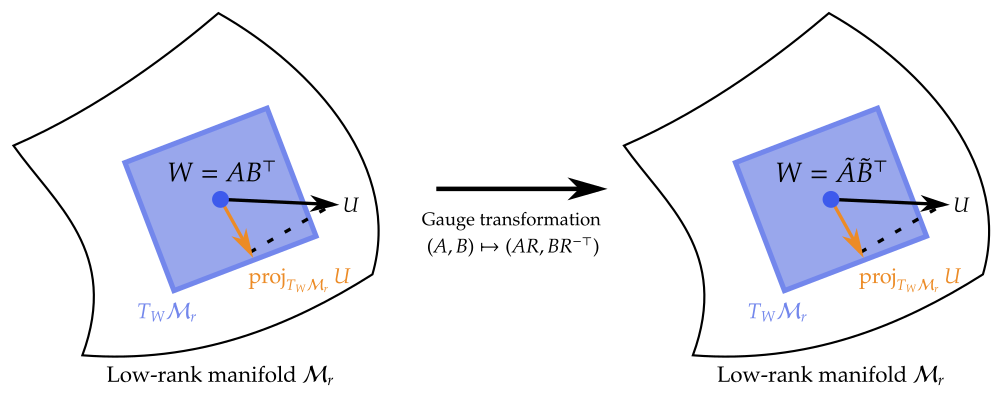}
    \caption{Our LoRA-Muon update implicitly applies the Muon update of the product $W=AB^T$ and projects it onto the tangent space of the low rank manifold. As a result, our update is invariant under gauge transformations of the LoRA factors.}
    \label{fig:schematic}
\end{figure}




\section{Preliminaries}\label{sec:preliminaries}
This section fixes the ambient steepest-descent setup, the Muon specialization, and the low-rank
geometry used in the derivation.

\subsection{Steepest descent under a norm}\label{subsec:sd-norm}
Let $f : \R^{m \times n} \to \R$ be differentiable and bounded below, let $W_t \in \R^{m \times n}$ be the
weight matrix at step $t$, let $G_t = \nabla f(W_t)$, let $\eta > 0$, and let $\norm{\cdot}$ be any norm.
A first-order Taylor expansion gives,
\begin{equation}
  f(W_t + \Delta W_t) = f(W_t) + \ip{G_t}{\Delta W_t} + o(\norm{\Delta W_t}). \label{eq:sd-general}
\end{equation}

The goal is to find an update $\Delta W_t$ that minimizes the LHS. Replacing the higher-order remainder by a hard trust-region constraint then yields the constrained linearized subproblem \cite{boydConvexOptimization2004},
\begin{equation}
  \begin{aligned}
    \Delta W_t^\star
    &=
    \argmin_{\Delta W_t \in \R^{m \times n}}
    \ip{G_t}{\Delta W_t}
    \quad
    \text{s.t.}
    \quad
    \norm{\Delta W_t} \le \eta \\
    &=
    - \eta U_t^\star,
    \qquad
    U_t^\star \in \argmax_{\norm{U} \le 1} \ip{G_t}{U}.
  \end{aligned}
  \label{eq:trust-region}
\end{equation}
As emphasized in \cite{bernstein2024oldnorm}, the norm controls direction while the scalar $\eta$
controls radius. We take this constrained form as the ambient starting point, because Muon and
LoRA-Muon are most naturally described by their linear minimization oracles on a norm ball
\cite{pethick2025normconstrainedlmos}.

\subsection{Muon as steepest descent under the spectral norm}\label{subsec:ambient-muon}
Take now $W_t \in \R^{m \times n}$ and equip matrix space with the spectral norm
$\norm{\cdot}_{2 \to 2}$. Then the linear minimization oracle over the spectral norm ball is the matrix
sign function,
\begin{equation}
  \argmin_{\norm{\Delta W_t}_{2 \to 2} \le \eta} \ip{G_t}{\Delta W_t}
  =
  - \eta \, \msign(G_t).
\end{equation}
For a rectangular SVD $G_t = U\Sigma V^\T$, we use the standard convention
$\msign(G_t)=UV^\T$.
Muon therefore performs steepest descent under the spectral norm in ambient matrix space. Momentum,
layer-wise scalings, and numerical orthogonalization matter in practice, but they sit on top of this
geometric core.

\subsection{LoRA as a low-rank optimization problem}\label{sec:lora-low-rank-problem}

In the rank-$r$ setting, we can write weights (and weight updates) as $W = AB^\T$ with $A \in \R^{m \times r}$ and $B \in \R^{n \times r}$ \cite{hu2022lora}.
That factorization is not unique: many different pairs $(A,B)$ produce the same low-rank matrix $W$. A good low-rank optimizer should therefore depend on the induced update on $W$, not on arbitrary coordinates in factor space.

Fix a rank $r \ll \min(m, n)$ and define the space of rank $r$ matrices,
\begin{equation}
  \Mcal_r = \{W \in \R^{m \times n} : \rank(W) = r\}.
\end{equation}
This is a differential manifold of matrices \cite{boumal2023intromanifolds}, and at a point $W = AB^\T$ with full-column-rank factors
$A \in \R^{m \times r}$ and $B \in \R^{n \times r}$, the tangent space is
\begin{equation}
  T_W \Mcal_r
  =
  \{\Delta A \, B^\T + A \, \Delta B^\T :
  \Delta A \in \R^{m \times r}, \ \Delta B \in \R^{n \times r}\}.
  \label{eq:tangent}
\end{equation}
The key observation is that a LoRA analogue of Muon should start from the tangent space \eqref{eq:tangent} itself and solve the steepest descent problem there, rather than first optimizing $A$ and $B$ as if they were independent coordinates.


\section{Deriving LoRA-Muon}\label{sec:derive-lora-muon}
\subsection{The constrained low-rank subproblem}\label{subsec:constrained-subproblem}

Let $f$ denote the objective over the offset low-rank family $W_{\mathrm{pre}} + \Mcal_r$. $W_{\mathrm{pre}}$ is the 'pretrained' weights which we can set to $0$ when training from scratch.
At $W_t = A_t B_t^\T$ with gradient $G_t = \nabla_W f(W_{\mathrm{pre}} + W_t)$, let $\norm{\cdot}$ be any unitarily invariant matrix norm, and define the idealized low-rank steepest-descent problem,
\begin{equation}
  \Delta W_t^\star
  =
  \argmin_{\Delta W_t \in T_{W_t}\Mcal_r}
  \ip{G_t}{\Delta W_t}
  \quad
  \text{s.t.}
  \quad
  \norm{\Delta W_t} \le \eta.
  \label{eq:lora-main}
\end{equation}
Writing $\Delta W_t = \Delta A B^\T + A \Delta B^\T$ makes the constraint difficult because the two
terms are coupled. A tractable approximation is to split the shared trust-region budget evenly between
the two tangent components and solve the resulting half-radius subproblems independently:
\begin{align}
  \Delta A^\star
  &=
  \argmin_{\Delta A}
  \ip{G_t}{\Delta A B^\T}
  \quad
  \text{s.t.}
  \quad
  \norm{\Delta A B^\T} \le \frac{\eta}{2},
  \label{eq:lora-decoupled-a}
  \\
  \Delta B^\star
  &=
  \argmin_{\Delta B}
  \ip{G_t}{A \Delta B^\T}
  \quad
  \text{s.t.}
  \quad
  \norm{A \Delta B^\T} \le \frac{\eta}{2}.
  \label{eq:lora-decoupled-b}
\end{align}
We write the subproblem in terms of the ambient gradient $G_t = \nabla_W f$, but LoRA training
does not expose $G_t$ directly through backpropagation. Instead, backprop gives us only the factor
gradients $\nabla_A f = G_t B$ and $\nabla_B f = G_t^\T A$. We therefore rewrite the weight-space
problem in those factor-gradient coordinates.

By the triangle inequality, any pair of feasible solutions to
\eqref{eq:lora-decoupled-a}--\eqref{eq:lora-decoupled-b} satisfies,
\begin{equation}
  \norm{\Delta A^\star B^\T + A \Delta B^{\star\T}}
  \le
  \norm{\Delta A^\star B^\T} + \norm{A \Delta B^{\star\T}}
  \le
  \eta,
\end{equation}
so the decoupled heuristic still respects the original trust-region radius.
This equal split is not merely a bookkeeping convenience. In the spectral specialization we study
experimentally, the two tangent
components $\Delta A B^\T$ and $A \Delta B^\T$ rapidly become strongly aligned and remain so
throughout training; see \hyperref[fig:alignment]{Figure~\ref*{fig:alignment}}. Here alignment means
the triangle-tightness ratio,
\begin{equation}
  \text{LoRA factor alignment}
  = 
  \frac{\norm{\Delta A B^\T + A\Delta B^\T}_{2 \to 2}}
  {\norm{\Delta A B^\T}_{2 \to 2} + \norm{A\Delta B^\T}_{2 \to 2}},
\end{equation}
which equals $1$ when the two tangent components add without cancellation under the spectral norm.
Once those two terms align, the triangle inequality above is close to tight, so dividing the
trust-region budget into two half-radius subproblems is a good approximation rather than an arbitrary
one.


\begin{figure}[t]
  \centering
  \includegraphics[width=0.8\linewidth]{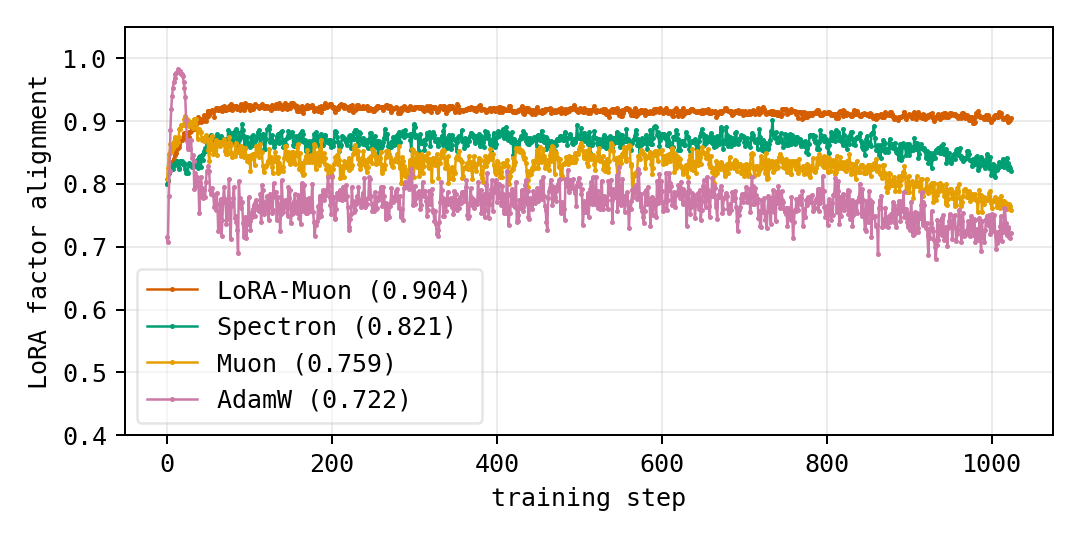}
  \caption{Alignment of the two tangent components $\Delta A B^\T$ and $A\Delta B^\T$ across
  training in the TinyShakespeare LoRA-Muon study. The components rapidly become and remain
  strongly aligned, so the triangle-inequality relaxation behind the $\eta/2$ split stays nearly tight
  throughout training rather than loosening over time.}
  \label{fig:alignment}
\end{figure}

\subsection{LMOs for steepest descent under unitary-invariant norms}\label{subsec:projector-form}

Let $P_A = A(A^\T A)^{-1}A^\T$ and $P_B = B(B^\T B)^{-1}B^\T$.
And following \cite{pethick2025normconstrainedlmos}, write,
\begin{equation}
  \operatorname{LMO}_{\norm{\cdot}}(X)
  \in
  \argmin_{\norm{U} \le 1} \ip{X}{U}.
\end{equation}

Unitary invariance lets the two decoupled constraints in \eqref{eq:lora-decoupled-a}--\eqref{eq:lora-decoupled-b} collapse onto the projector terms.
Let $S_A = A^\T A$ and $S_B = B^\T B$. If $B = \widetilde{B} S_B^{1/2}$ with $\widetilde{B}^\T \widetilde{B} = I$, then $\norm{\Delta A B^\T} = \norm{\Delta A S_B^{1/2}\widetilde{B}^\T} = \norm{\Delta A S_B^{1/2}}$,
and the $A$-side objective depends only on the column projector $P_B$; the $B$ side is identical
with $A$ and $B$ swapped.
Then the two decoupled subproblems admit especially clean projector-form solutions for any
unitarily invariant norm:
\begin{equation}
  \begin{aligned}
    \Delta A^\star B^\T
    &=
    \frac{\eta}{2}\operatorname{LMO}_{\norm{\cdot}}(G_t P_B),
    \\
    A \Delta B^{\star\T}
    &=
    \frac{\eta}{2}\operatorname{LMO}_{\norm{\cdot}}(P_A G_t).
  \end{aligned}
  \label{eq:projector-lmo}
\end{equation}
Equation \eqref{eq:projector-lmo} is the key formula in the derivation. It already shows both the
connection to Muon and the connection to gauge symmetry. If $P_A = P_B = I$,
we recover ambient Muon. If gauge transformations leave $P_A$ and $P_B$ unchanged, then the
induced update on $W$ is unchanged as well.

\subsection{Spectral specialization: LoRA-Muon}\label{subsec:closed-form-factor}
LoRA-Muon is the spectral specialization of \eqref{eq:projector-lmo}. Setting
$\norm{\cdot} = \norm{\cdot}_{2 \to 2}$ gives,
$ \operatorname{LMO}_{\norm{\cdot}_{2 \to 2}}(X) = -\msign(X) $,
so the projector-form update becomes,
\begin{equation}
  \begin{aligned}
    \Delta A^\star B^\T &= -\frac{\eta}{2}\msign(G_t P_B), \\
    \qquad
    A \Delta B^{\star\T} &= -\frac{\eta}{2}\msign(P_A G_t).
  \end{aligned}
  \label{eq:projector-form}
\end{equation}
Let $S_A = A^\T A$ and $S_B = B^\T B$. Using
$\nabla_A f = G_t B$ and $\nabla_B f = G_t^\T A$, one convenient exact factorization of
\eqref{eq:projector-form} is
\begin{align}
  \Delta A^\star
  &=
  -\frac{\eta}{2}\,
  \msign\!\big(\nabla_A f \, S_B^{-1/2}\big)\,
  S_B^{-1/2},
  \label{eq:delta-a}
  \\
  \Delta B^\star
  &=
  -\frac{\eta}{2}\,
  \msign\!\big(\nabla_B f \, S_A^{-1/2}\big)\,
  S_A^{-1/2}.
  \label{eq:delta-b}
\end{align}
This gives the ideal LoRA-Muon update. In words, we whiten the factor gradients by the current Gram
geometry and then apply the same spectral descent rule that ambient Muon applies in full matrix
space. \appendixref{subsec:appendix-direct} derives this update directly from the decoupled
subproblems, and \appendixref{subsec:appendix-orthog} together with
\appendixref{subsec:appendix-invroot} give the numerical realization.

\begin{algorithm}[t]
\caption{LoRA-Muon}
\label{alg:lora-muon}
\small
\begin{algorithmic}[1]
\REQUIRE Factors $A_t \in \R^{m \times r}$, $B_t \in \R^{n \times r}$, previous first moments $m_{t-1}^A \in \R^{m \times r}$, $m_{t-1}^B \in \R^{n \times r}$, learning rate $\eta_t$, momentum $\beta$, weight decay $\lambda$
\STATE $g_t^A \gets \nabla_A f(W_{\mathrm{pre}} + A_t B_t^\T)$ and $g_t^B \gets \nabla_B f(W_{\mathrm{pre}} + A_t B_t^\T)$
\STATE $m_t^A \gets \beta m_{t-1}^A + (1-\beta) g_t^A$ and $m_t^B \gets \beta m_{t-1}^B + (1-\beta) g_t^B$
\STATE $S_A \gets A_t^\T A_t$ and $S_B \gets B_t^\T B_t$
\STATE $R_A, R_B \gets \texttt{matrix\_invroot}(\operatorname{stack}([S_A, S_B]))$
\STATE $\Delta A_t \gets -\frac{\eta_t}{2}\msign(m_t^A R_B) R_B$
\STATE $\Delta B_t \gets -\frac{\eta_t}{2}\msign(m_t^B R_A) R_A$
\STATE $s_t \gets \sqrt{1 - \lambda \eta_t}$
\STATE $A_{t+1} \gets s_t A_t + s_t^{-1}\Delta A_t$
\STATE $B_{t+1} \gets s_t B_t + s_t^{-1}\Delta B_t$
\STATE \textbf{return} $A_{t+1}, B_{t+1}, m_t^A, m_t^B$
\end{algorithmic}
\end{algorithm}

\subsection{Split weight decay}\label{subsec:split-weight-decay}

Applying decoupled weight decay factor-wise yields updates of the form,
\begin{align}
  W_{t+1}^{\text{naive}}
  &= ((1 - \lambda\eta_t)A_t + \Delta A_t^\star) ((1 - \lambda\eta_t)B_t + \Delta B_t^\star)^T \\
  &= (1 - \lambda\eta_t)^2 A_t B_t^T + (1 - \lambda\eta_t) ( A_t (\Delta B_t^\star)^T + \Delta A_t^\star B_t^T) + \Delta A_t^\star (\Delta B_t^\star)^T.
\end{align}
But for training dynamics in the full-rank and low-rank settings to match (and therefore allow us to transfer hyperparameters), we must apply the (decoupled) weight decay term on the composed weight $W = A B^T \in \mathcal{M}_r$:
\begin{equation}
    W_{t+1}^{\text{expected}}
    = (1 - \lambda\eta_t) \underbrace{A_t B_t^T}_{W_t} + \underbrace{A_t (\Delta B_t^\star)^T + \Delta A_t^\star B_t^T}_{\Delta W_t}.
\end{equation}
To fix this mismatch, we set $s = \sqrt{1 - \lambda \eta_t}$ and instead update LoRA factors as follows,
\begin{equation}
  A_{t+1} = s A_t + \frac{1}{s}\Delta A_t^\star,
  \qquad
  B_{t+1} = s B_t + \frac{1}{s}\Delta B_t^\star.
  \label{eq:split-wd}
\end{equation}
This results in updates of the form,
\begin{equation}
    W_{t+1}^{\text{split-weight-decay}}
    =
    W_{t+1}^{\text{expected}}
    + \frac{1}{s^2} \Delta A_t^\star (\Delta B_t^\star)^T
\end{equation}
where the second (error) term approaches $0$ as we decay the learning rate. \appendixref{subsec:appendix-splitwd} shows that \eqref{eq:split-wd} preserves the gauge-invariance property which we will discuss in the next secion.

The closest prior ideas we know are Frobenius decay for factorized layers, which regularizes the product $AB^\T$ directly \cite{khodak2021initialization}. To the best of our knowledge, however, prior LoRA optimizers do not implement that correction as the split decoupled update \eqref{eq:split-wd}.

Algorithm~\ref{alg:lora-muon} gives the LoRA-Muon step. This is the LoRA-Muon update analyzed and used throughout the paper. 

\section{Gauge Symmetry}\label{sec:gauge}

This section proves that the same projector-form update is invariant to the
arbitrary choice of LoRA factorization.

\subsection{Gauge action}\label{subsec:gauge-action}
Two factor pairs $(A, B)$ and $(A', B')$ represent the same low-rank matrix if
\begin{equation}
  A' = AR,
  \qquad
  B' = BR^{-\T},
  \qquad
  R \in \GL(r).
\end{equation}
The scalar case $R = cI$ is the rescaling used in our experiments and in the gauge-rebalancing
algorithm, but the theory below is formulated for the full gauge group.

\begin{proposition}[Gauge invariance of the projectors]
Let $A \in \R^{m \times r}$ and $B \in \R^{n \times r}$ have full column rank, and define
\begin{equation}
  P_A := A(A^\T A)^{-1}A^\T,
  \qquad
  P_B := B(B^\T B)^{-1}B^\T.
\end{equation}
Then for every $R \in \GL(r)$, the transformed factors $A' = AR$ and $B' = BR^{-\T}$ satisfy
$P_{A'} = P_A$ and $P_{B'} = P_B$.
\end{proposition}

\begin{proof}
For $A' = AR$, $P_{A'} = AR(R^\T A^\T A R)^{-1}R^\T A^\T = A(A^\T A)^{-1}A^\T = P_A$. The proof for $P_B$ is identical.
\end{proof}

\subsection{Gauge invariance of the update}\label{subsec:gauge-update}
\begin{theorem}[Gauge invariance of the induced projector-form LMO update]
Let $A' = AR$ and $B' = BR^{-\T}$ with $R \in \GL(r)$, and let $W = AB^\T = A'B'^\T$. Define
\begin{equation}
  P_A := A(A^\T A)^{-1}A^\T,
  \qquad
  P_B := B(B^\T B)^{-1}B^\T,
\end{equation}
let $G_t = \nabla_W f(W_{\mathrm{pre}} + W)$, and fix any deterministic choice of
$\operatorname{LMO}_{\norm{\cdot}}(\cdot)$ in the projector-form update derived in
\secref{sec:derive-lora-muon},
\begin{equation}
  \Delta W
  =
  \frac{\eta}{2}\operatorname{LMO}_{\norm{\cdot}}(G_t P_B)
  +
  \frac{\eta}{2}\operatorname{LMO}_{\norm{\cdot}}(P_A G_t),
\end{equation}
which is exactly \eqref{eq:projector-lmo}. Then the induced $\Delta W$ is the same whether it is
computed from $(A,B)$ or from $(A',B')$. In particular, the spectral specialization
$\operatorname{LMO}_{\norm{\cdot}_{2 \to 2}}(X) = -\msign(X)$ yields gauge invariance of the ideal
LoRA-Muon update \eqref{eq:projector-form}.
\end{theorem}

Section~\ref{subsec:projector-form} already used unitary invariance to derive
\eqref{eq:projector-lmo}. Once that form is in hand, the gauge claim below follows only from the
fact that the projector arguments do not change across equivalent factorizations.

\begin{proof}
By the previous proposition, the projectors are gauge-invariant. Since
$W = AB^\T = A'B'^\T$, the ambient gradient $G_t = \nabla_W f(W_{\mathrm{pre}} + W)$ is also
unchanged by the gauge transformation. Therefore the two oracle inputs $G_t P_B$ and $P_A G_t$
are themselves gauge-invariant. Applying the same deterministic
$\operatorname{LMO}_{\norm{\cdot}}$ to the same two arguments therefore returns the same two tangent
terms before and after gauge transformation, so their sum $\Delta W$ is unchanged as an element of
$T_W \Mcal_r$.
\end{proof}

For the spectral specialization, an exact factor-form equivariance proof is given in
\appendixref{subsec:appendix-factor-equiv}.

The same symmetry also justifies scalar gauge rebalancing as a numerical conditioning tool. We can
choose a better-conditioned representative of the same factorization class without changing the
underlying LoRA-Muon step on $W$, provided the first moments are transported with the factors.
\appendixref{subsec:appendix-gauge-rebalancing} gives the rebalancing rule, and
\appendixref{subsec:appendix-moments} proves the moment-transport identity. Empirically,
Figure~\ref{fig:transfer-sweeps-grid} checks the corresponding stability claim: LoRA-Muon's
learning-rate curve remains nearly unchanged under large scalar factor rescalings, whereas Spectron
is sensitive to the same arbitrary choice of representative.

\section{Spectron and LoRA-RITE Through the Same Lens}\label{sec:spectron-rite}

We now compare LoRA-Muon against the closest low-rank optimizers in the draft: Spectron and
LoRA-RITE. Two points matter most. Spectron is not gauge-invariant, so its behavior can change
under arbitrary factor rescalings; \appendixref{subsec:appendix-spectron} proves this with an
explicit scalar counterexample. LoRA-RITE's simplified QR-coordinate core, by contrast, realizes
the same spectral update as LoRA-Muon in different coordinates.

\subsection{Spectron as spectral renormalization in factor coordinates}\label{subsec:spectron}
Spectron \cite{janson2026spectron} is the closest baseline, but it chooses its trust radius from
the current factor norms rather than from the quotient geometry of the induced low-rank update.
Consequently, arbitrary gauge rescalings of the same matrix can change Spectron's realized step
size, while LoRA-Muon leaves the induced update on $W$ unchanged. We therefore view Spectron as a
spectral, factor-coordinate renormalization heuristic rather than the gauge-invariant spectral
steepest-descent update studied here; \appendixref{subsec:appendix-spectron-factor-coordinates}
gives the detailed radius formula, gauge-sensitivity discussion, and reference update.

\subsection{LoRA-RITE as the same spectral update in QR coordinates}\label{subsec:lorarite}
The simplified QR-coordinate core of LoRA-RITE realizes the same spectral steepest-descent update
as LoRA-Muon, but writes it in QR coordinates rather than in Gram-matrix or projector form. This
identifies a shared ideal spectral core, not an identity between the full optimizers: LoRA-Muon is
QR-free and first-moment-only, while the original LoRA-RITE optimizer carries additional transported
second-moment and escaped-mass state. \appendixref{subsec:appendix-lorarite-qr} gives the
QR-coordinate derivation and proves the algebraic equivalence to the LoRA-Muon factor update.

\section{Experiments}\label{sec:experiments}
Every quantitative result here comes from TinyShakespeare character-level language modeling on the
corpus popularized by
\cite{karpathy2015charrnn}.

We train causal Transformer language models in the original Transformer style
\cite{vaswani2017attention}, but with the concrete architectural choices used in our experiments:
bias-free linear layers; self-attention with query-key normalization in the style of
\cite{henry2020qknorm}, RoPE \cite{su2021roformer}, and FlexAttention \cite{dong2024flexattention};
GELU MLPs \cite{hendrycks2016gelu} with $4\times$ hidden-width expansion; RMS normalization
without affine parameters \cite{zhang2019rmsnorm}; and residual scaling $\alpha = 1/(2L)$
following the modular-norm-style parametrization of \cite{large2024modular}. Each main sweep
uses a dense reference budget of roughly 1M tokens. For the LoRA sweeps, we compute the FLOP cost
of the dense and low-rank runs and increase the low-rank step budget as needed so every LoRA run is
compute-matched to the dense baseline. We apply LoRA methods only to the linear layers in the Transformer backbone,
and the rest with (dense) AdamW \cite{kingma2015adam,loshchilov2019adamw}.
We sweep over learning rates [0.01, 1.0] on a log scale for both the dense and LoRA runs, and we select the best tested learning rate by lowest seed-mean validation loss for each plotted rank point. We report standard errors across six seeds per plotted rank point.
The exact compute-match multipliers appear in
\appendixref{subsec:appendix-flops}.

\begin{figure}[tbp]
  \centering
  \begin{minipage}{0.49\linewidth}
    \centering
    \includegraphics[width=\linewidth]{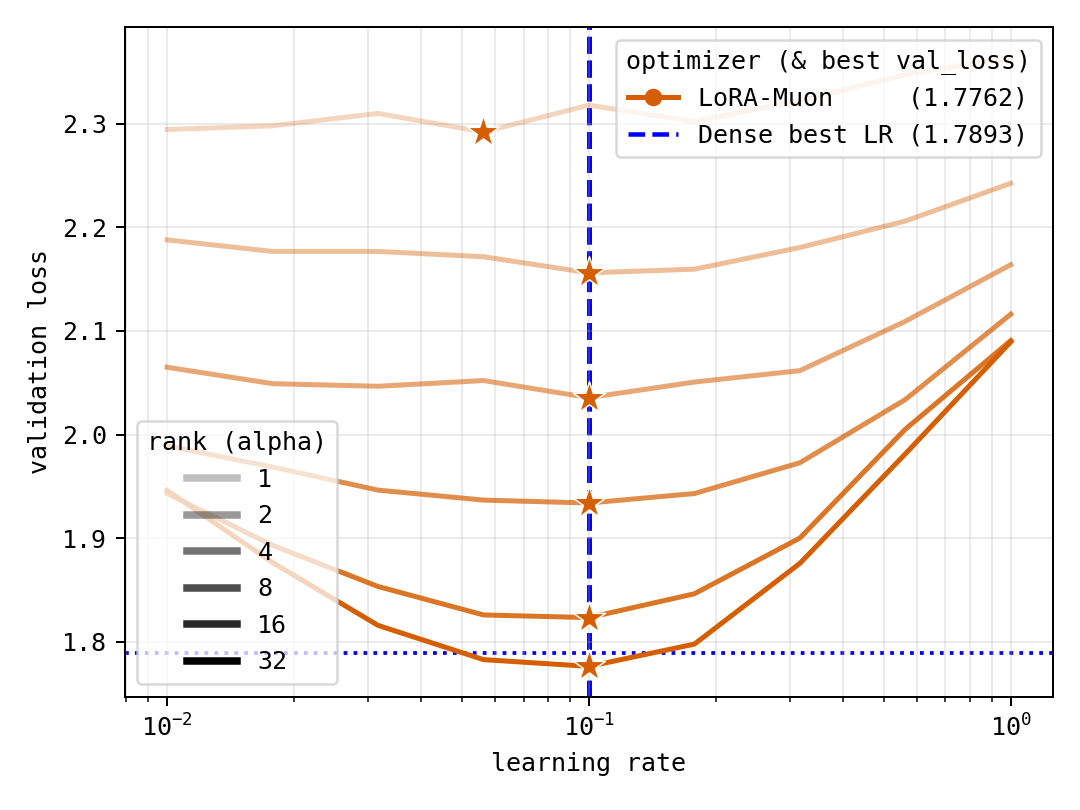}\\[-0.5ex]
    \small Rank sweep
  \end{minipage}\hfill
  \begin{minipage}{0.49\linewidth}
    \centering
    \includegraphics[width=\linewidth]{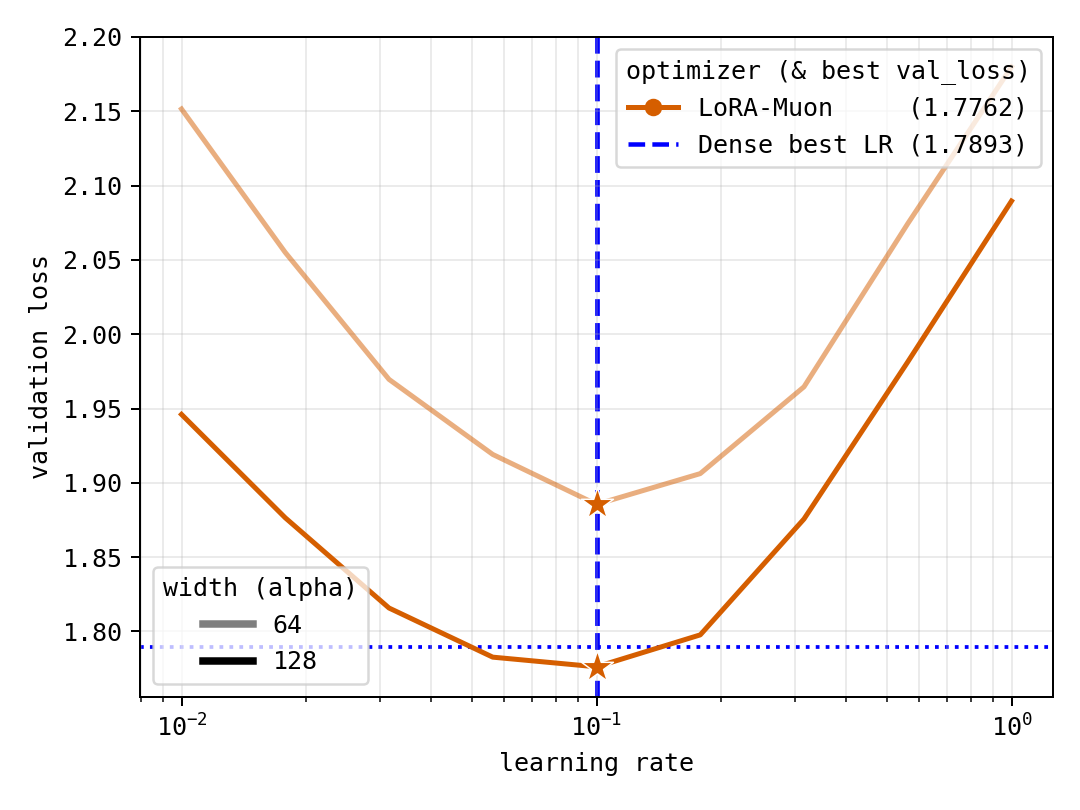}\\[-0.5ex]
    \small Width sweep
  \end{minipage}

  \vspace{1.0ex}

  \begin{minipage}{0.49\linewidth}
    \centering
    \includegraphics[width=\linewidth]{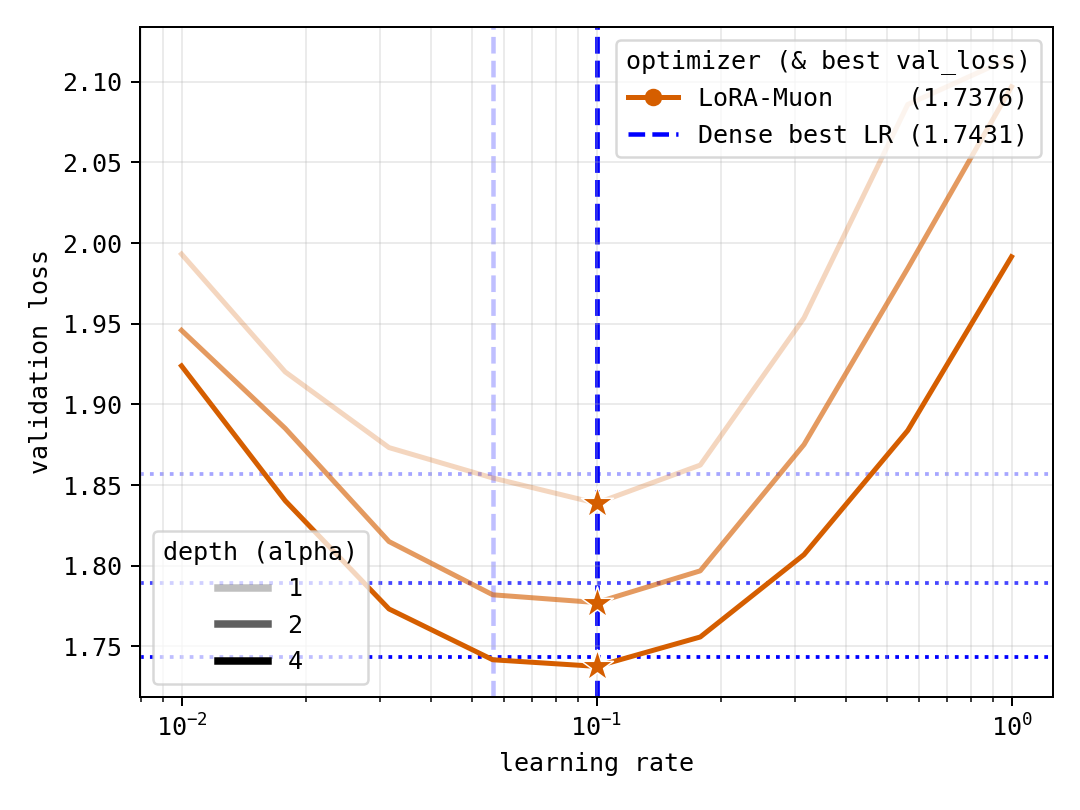}\\[-0.5ex]
    \small Depth sweep
  \end{minipage}\hfill
  \begin{minipage}{0.49\linewidth}
    \centering
    \includegraphics[width=\linewidth]{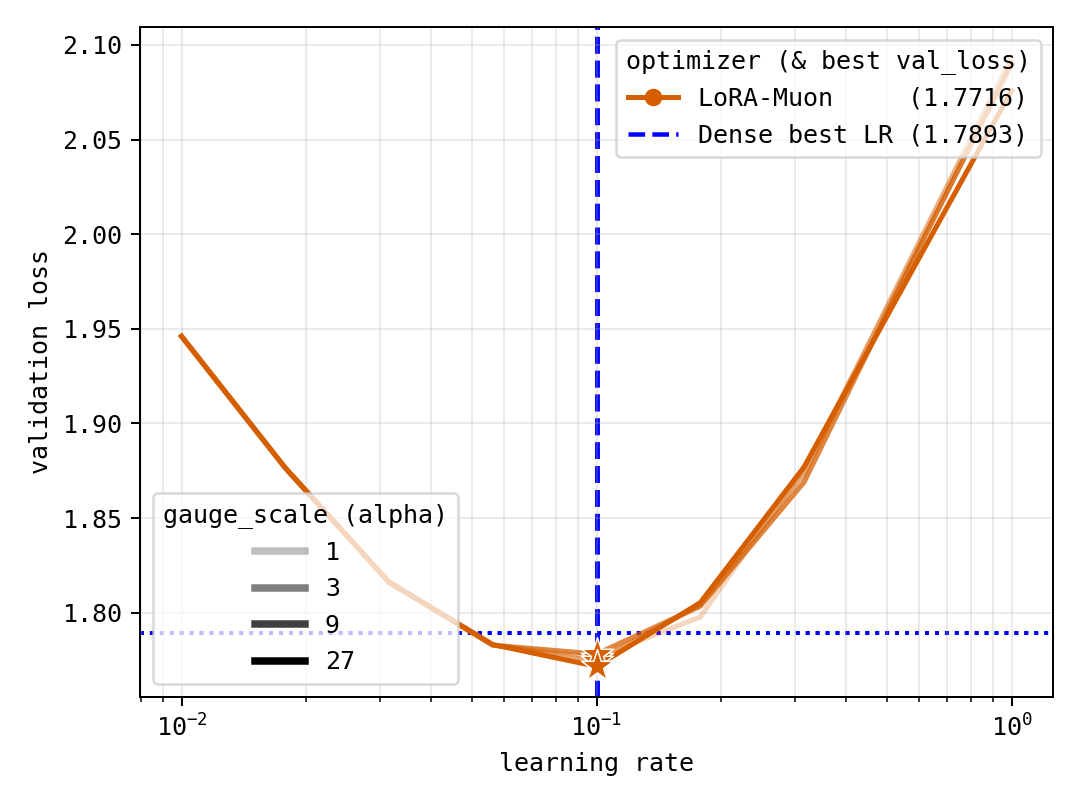}\\[-0.5ex]
    \small Initial scalar gauge-scale sweep
  \end{minipage}
  \caption{Validation loss versus learning rate across the four learning-rate-transfer sweeps: rank,
  width, depth, and scalar gauge scale. LoRA-Muon matches dense Muon's best tested learning rate
  across all four sweeps, making it a cheap proxy for hyperparameter tuning before running
  larger-scale training runs.}
  \label{fig:transfer-sweeps-grid}
\end{figure}

\subsection{Learning-rate transfer across rank, width, depth, and gauge scale}
Figure~\ref{fig:transfer-sweeps-grid} summarizes the four learning-rate-transfer sweeps. Across the
experiments reported here, the dense best tested learning rate remains at
$\texttt{lr\_linear}=0.1$, and LoRA-Muon matches that rate across rank, width, depth, and scalar
factor rescaling. In the rank sweep, the match appears already at rank $2$ and continues through
ranks $4, 8, 16,$ and $32$, with rank $1$ as the lone clear failure case. At rank $2$, LoRA-Muon
selects the same best tested learning rate as dense, with validation loss $2.156 \pm 0.005$ compared
to the dense $1.789 \pm 0.002$. At rank $32$, LoRA-Muon still selects the dense best tested learning
rate and reaches $1.776 \pm 0.002$, below the dense $1.789 \pm 0.002$ in this six-seed sweep.

The width and depth sweeps fix the LoRA rank at $32$; in these TinyShakespeare sweeps, rank-$32$
LoRA-Muon is competitive with the dense references and is lower in the depth sweep. The
factor-rescaling sweep gives the complementary gauge-invariance check: the LoRA-Muon curves stay
nearly coincident across scalar gauge choices and keep the same best tested learning rate. By contrast,
Spectron often prefers a larger learning rate around $0.178$, which is consistent with it optimizing in
a different, gauge-sensitive local geometry. In the matched $d_{\mathrm{model}}=128$,
$n_{\mathrm{layers}}=2$ sweep configuration, the rank-$2$ proxy uses only about $4.3\%$ as many
trainable parameters as the dense model and supports a $9.225\times$ larger optimizer-step budget
under the same 1M-token compute budget. The full rank-by-rank table and dense-matrix FLOP
breakdown appear in \appendixref{subsec:appendix-flops}.

\subsection{Gauge rebalancing}

\begin{figure}[tbp]
  \centering
  \begin{minipage}{0.48\linewidth}
    \centering
    \includegraphics[width=\linewidth]{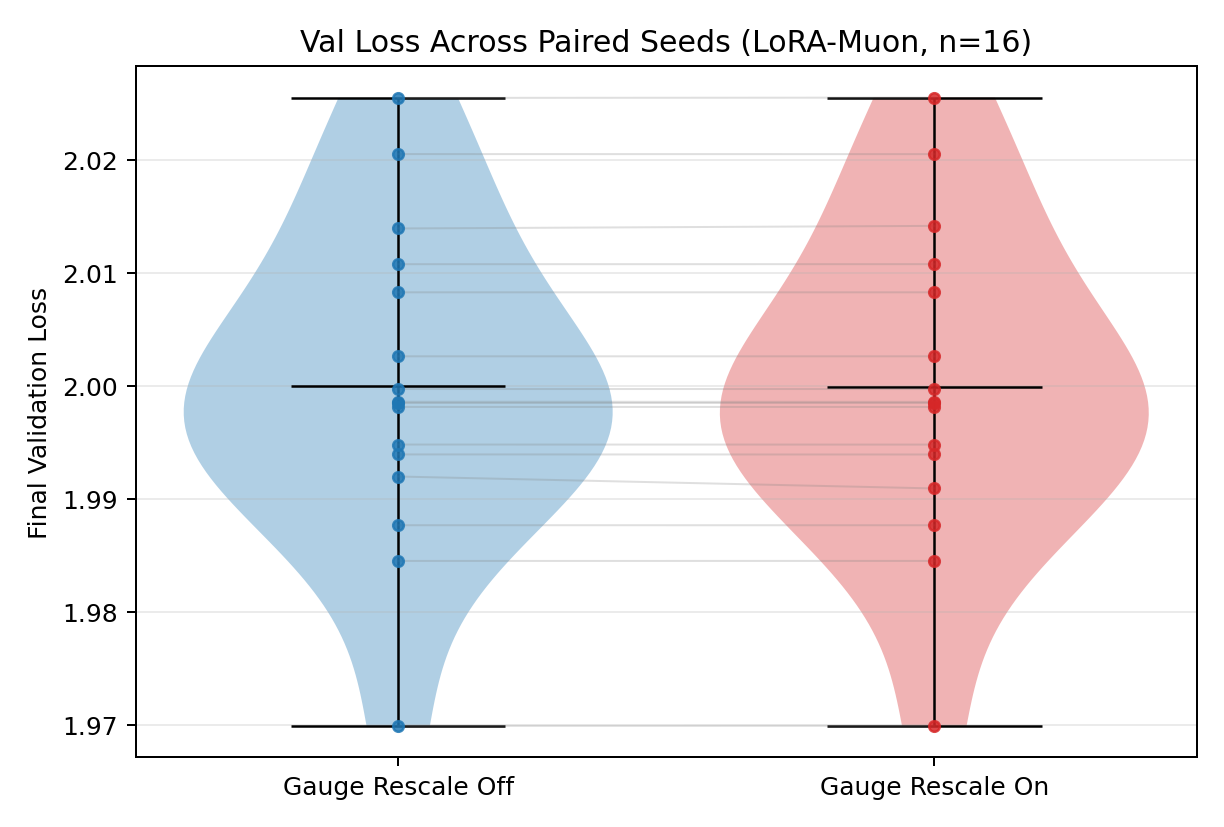}\\[-0.5ex]
  \end{minipage}\hfill
  \begin{minipage}{0.48\linewidth}
    \centering
    \includegraphics[width=\linewidth]{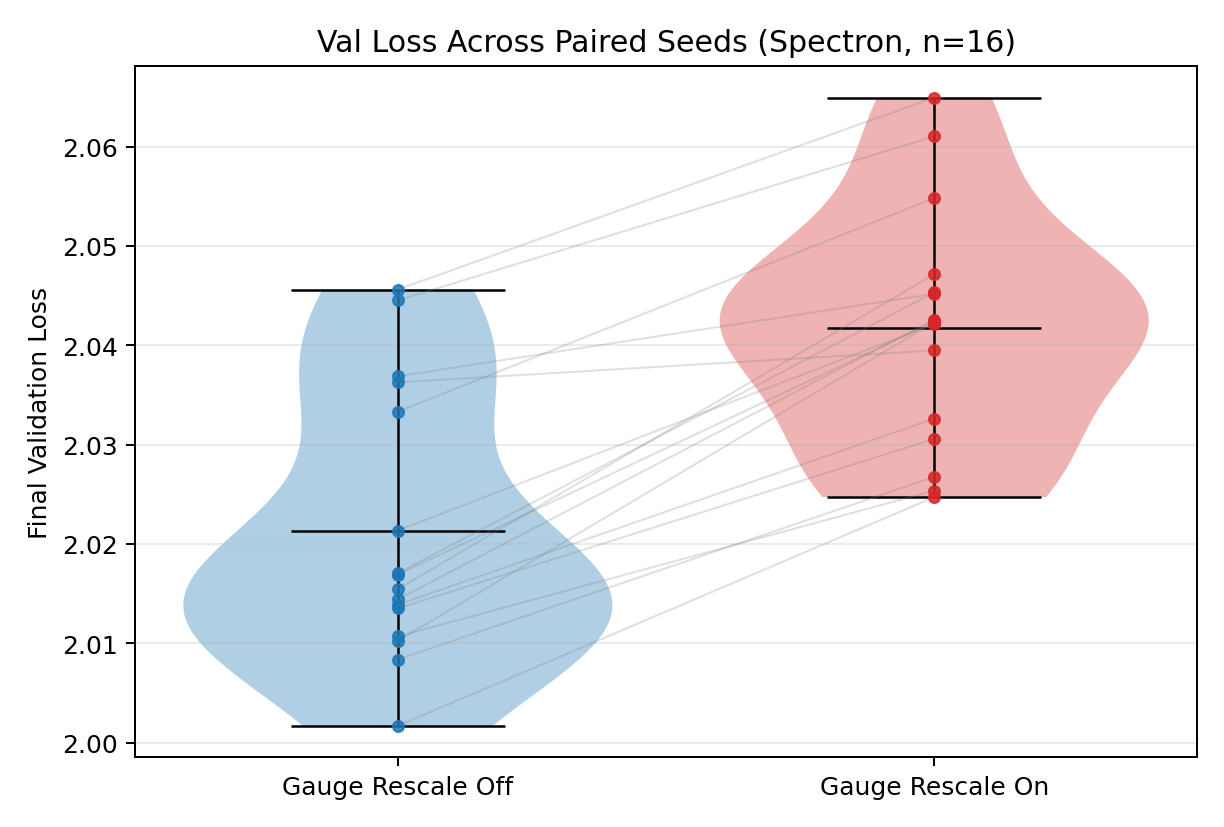}\\[-0.5ex]
  \end{minipage}
  \caption{Paired validation-loss comparisons with scalar gauge rebalancing disabled versus enabled
  for LoRA-Muon and Spectron after initializing the random LoRA factors at gauge scale $9$
  by multiplying one factor and dividing the other. LoRA-Muon's loss distribution is essentially
  unchanged by rebalancing, while Spectron's changes substantially, consistent with
  \secref{subsec:spectron}.}
  \label{fig:gauge-rebalance-violins}
\end{figure}

Here we test whether scalar gauge rebalancing acts as a harmless conditioning step or as an
optimizer-changing intervention. We initialize the random LoRA factors at gauge scale $9$ by
multiplying one factor and dividing the other, then compare paired runs with rebalancing disabled
versus enabled every step. In Figure~\ref{fig:gauge-rebalance-violins}, LoRA-Muon's mean
validation-loss change is $-5.1 \times 10^{-5}$ with a $95\%$ paired confidence interval
$[-2.0 \times 10^{-4}, 9.5 \times 10^{-5}]$, consistent with the theory that rebalancing does not
change the induced weight-space update. For Spectron, the same intervention changes validation loss
by $2.0 \times 10^{-2}$ on average, with paired confidence interval
$[1.6 \times 10^{-2}, 2.5 \times 10^{-2}]$. Thus the same factor-coordinate operation is nearly
invisible to LoRA-Muon but visible to Spectron, matching the gauge-sensitivity analysis in
\secref{subsec:spectron}.


\section{Limitations}\label{sec:limitations}
First, the experiments are intentionally small:
the empirical claims are established on TinyShakespeare-scale Transformer language models, not yet
on large-scale pretraining or downstream finetuning benchmarks. The geometric derivation is not tied
to this dataset, but the reported transfer behavior should be read as evidence for this tested regime.
Second, the direct sweeps test learning-rate transfer, especially for $\texttt{lr\_linear}$, rather than
full hyperparameter transfer. Prior work, however, show that the optimal learning rate is coupled with
other hyperparameters such as weight decay \cite{kosson2026weightdecay}, batch size, training horizon,
and momentum \cite{shulgin2026hyperscaling,islamov2026batchsize} in a predictable way.

\section{Conclusion}
LoRA-Muon is obtained by applying Muon's spectral steepest-descent principle directly on the
low-rank manifold. This viewpoint gives a simple projector-form update, explains why the update is
invariant to arbitrary LoRA factorization choices, and separates LoRA-Muon from nearby
factor-coordinate methods such as Spectron. It also clarifies the relationship to LoRA-RITE: the
simplified QR-coordinate core is the same spectral update written in a different coordinate system,
whereas LoRA-Muon realizes it without QR factorizations or transported second moments.

Empirically, the resulting optimizer makes low-rank training useful for dense Muon learning-rate
search. In the TinyShakespeare sweeps, compute-matched LoRA-Muon proxies recover the dense best
tested learning rate across rank, width, depth, and scalar factor rescaling, while gauge rebalancing
leaves LoRA-Muon essentially unchanged but visibly affects Spectron. The broader lesson is that the
low-rank parametrization should not determine the optimizer geometry: for LoRA, the useful update is
the one defined by the low-rank matrix manifold itself.


\newpage
\bibliographystyle{unsrtnat}
\bibliography{paper}


\newpage

\appendix

\section{Proofs}\label{sec:appendix-proofs}
\subsection{From the decoupled subproblems to the closed-form factor updates}\label{subsec:appendix-direct}
We derive \eqref{eq:delta-a}--\eqref{eq:delta-b} directly from the decoupled subproblems
\eqref{eq:lora-decoupled-a}--\eqref{eq:lora-decoupled-b}, without passing through
\eqref{eq:projector-form}.

\begin{proposition}[Direct derivation of the factor update]
Consider the decoupled spectral subproblems
\begin{equation}
  \Delta A^\star
  =
  \argmin_{\Delta A}
  \ip{G_t}{\Delta A B^\T}
  \quad
  \text{s.t.}
  \quad
  \norm{\Delta A B^\T}_{2 \to 2} \le \frac{\eta}{2},
\end{equation}
\begin{equation}
  \Delta B^\star
  =
  \argmin_{\Delta B}
  \ip{G_t}{A \Delta B^\T}
  \quad
  \text{s.t.}
  \quad
  \norm{A \Delta B^\T}_{2 \to 2} \le \frac{\eta}{2}.
\end{equation}
If $S_A = A^\T A$, $S_B = B^\T B$, $\nabla_A f = G_t B$, and $\nabla_B f = G_t^\T A$, then their
solutions are
\begin{equation}
  \Delta A^\star
  =
  -\frac{\eta}{2}\msign(\nabla_A f\, S_B^{-1/2}) S_B^{-1/2},
\end{equation}
\begin{equation}
  \Delta B^\star
  =
  -\frac{\eta}{2}\msign(\nabla_B f\, S_A^{-1/2}) S_A^{-1/2}.
\end{equation}
\end{proposition}

\begin{proof}
We prove the $A$ update first, starting from \eqref{eq:lora-decoupled-a}. Let
$S_B = B^\T B$ and define the orthonormalized factor
$\widetilde{B} := B S_B^{-1/2}$, so that $\widetilde{B}^\T \widetilde{B} = I$ and
$B^\T = S_B^{1/2}\widetilde{B}^\T$. Introduce the reparameterized variable
$Z_A := \Delta A S_B^{1/2} \in \R^{m \times r}$. Then
\begin{equation}
  \Delta A B^\T
  =
  \Delta A S_B^{1/2}\widetilde{B}^\T
  =
  Z_A \widetilde{B}^\T.
\end{equation}
Since $\widetilde{B}$ has orthonormal columns, unitary invariance of the spectral norm gives
\begin{equation}
  \norm{\Delta A B^\T}_{2 \to 2}
  =
  \norm{Z_A \widetilde{B}^\T}_{2 \to 2}
  =
  \norm{Z_A}_{2 \to 2}.
\end{equation}
Likewise, the objective in \eqref{eq:lora-decoupled-a} becomes
\begin{equation}
  \ip{G_t}{\Delta A B^\T}
  =
  \ip{G_t}{Z_A \widetilde{B}^\T}
  =
  \ip{G_t \widetilde{B}}{Z_A}
  =
  \ip{G_t B S_B^{-1/2}}{Z_A}.
\end{equation}
Therefore \eqref{eq:lora-decoupled-a} is equivalent to the rank-side problem
\begin{equation}
  Z_A^\star
  =
  \argmin_{Z_A \in \R^{m \times r}}
  \ip{G_t B S_B^{-1/2}}{Z_A}
  \quad
  \text{s.t.}
  \quad
  \norm{Z_A}_{2 \to 2} \le \frac{\eta}{2}.
\end{equation}
For the spectral norm, the linear minimization oracle is $-\msign$, so
\begin{equation}
  Z_A^\star
  =
  -\frac{\eta}{2}\msign(G_t B S_B^{-1/2}).
\end{equation}
Undoing the change of variables gives
\begin{equation}
  \Delta A^\star
  =
  Z_A^\star S_B^{-1/2}
  =
  -\frac{\eta}{2}\msign(G_t B S_B^{-1/2}) S_B^{-1/2}.
\end{equation}
Using $\nabla_A f = G_t B$ yields \eqref{eq:delta-a}.

For the $B$ update, start from \eqref{eq:lora-decoupled-b} and first rewrite the objective and
constraint in the equivalent transposed form
\begin{equation}
  \ip{G_t}{A \Delta B^\T}
  =
  \ip{G_t^\T}{\Delta B A^\T},
  \qquad
  \norm{A \Delta B^\T}_{2 \to 2}
  =
  \norm{\Delta B A^\T}_{2 \to 2}.
\end{equation}
Now set $S_A = A^\T A$, $\widetilde{A} := A S_A^{-1/2}$, and
$Z_B := \Delta B S_A^{1/2}$. Then $\Delta B A^\T = Z_B \widetilde{A}^\T$, so
\begin{equation}
  \norm{\Delta B A^\T}_{2 \to 2} = \norm{Z_B}_{2 \to 2},
  \qquad
  \ip{G_t^\T}{\Delta B A^\T}
  =
  \ip{G_t^\T A S_A^{-1/2}}{Z_B}.
\end{equation}
Thus \eqref{eq:lora-decoupled-b} is equivalent to
\begin{equation}
  Z_B^\star
  =
  \argmin_{Z_B \in \R^{n \times r}}
  \ip{G_t^\T A S_A^{-1/2}}{Z_B}
  \quad
  \text{s.t.}
  \quad
  \norm{Z_B}_{2 \to 2} \le \frac{\eta}{2},
\end{equation}
whose solution is
\begin{equation}
  Z_B^\star
  =
  -\frac{\eta}{2}\msign(G_t^\T A S_A^{-1/2}).
\end{equation}
Undoing the change of variables yields
\begin{equation}
  \Delta B^\star
  =
  -\frac{\eta}{2}\msign(G_t^\T A S_A^{-1/2}) S_A^{-1/2}.
\end{equation}
Using $\nabla_B f = G_t^\T A$ gives \eqref{eq:delta-b}.
\end{proof}

\subsection{Factor-form gauge equivariance}\label{subsec:appendix-factor-equiv}
The projector-form theorem in the body has an exact factor-form counterpart.

\begin{proposition}[Factor-form equivariance]
Let $A' = AR$ and $B' = BR^{-\T}$ with $R \in \GL(r)$, and define the ideal LoRA-Muon factor
increments by
\begin{equation}
  \Delta A
  =
  -\frac{\eta}{2}\msign(\nabla_A f\, S_B^{-1/2}) S_B^{-1/2},
  \qquad
  S_B = B^\T B,
\end{equation}
\begin{equation}
  \Delta B
  =
  -\frac{\eta}{2}\msign(\nabla_B f\, S_A^{-1/2}) S_A^{-1/2},
  \qquad
  S_A = A^\T A.
\end{equation}
Then the corresponding increments at $(A',B')$ satisfy
\begin{equation}
  \Delta A' = \Delta A \, R,
  \qquad
  \Delta B' = \Delta B \, R^{-\T}.
\end{equation}
\end{proposition}

\begin{proof}
We prove the $A$ statement; the $B$ statement is analogous. Write
$S_B = B^\T B$ and $S_B' = B'^\T B' = R^{-1} S_B R^{-\T}$. Define
\begin{equation}
  Q_B := S_B^{-1/2} R \, S_B'^{1/2}.
\end{equation}
Then
\begin{equation}
  Q_B Q_B^\T
  =
  S_B^{-1/2} R S_B' R^\T S_B^{-1/2}
  =
  S_B^{-1/2} R (R^{-1} S_B R^{-\T}) R^\T S_B^{-1/2}
  =
  I,
\end{equation}
so $Q_B$ is orthogonal. Re-arranging the definition gives
\begin{equation}
  S_B^{-1/2} R = Q_B S_B'^{-1/2},
  \qquad
  R^{-\T} S_B'^{-1/2} = S_B^{-1/2} Q_B^\T.
  \label{eq:qb-identities}
\end{equation}
Now the transformed gradient is $\nabla_A' f = \nabla_A f \, R^{-\T}$, so by
\eqref{eq:qb-identities},
\begin{equation}
  \nabla_A' f \, S_B'^{-1/2}
  =
  \nabla_A f \, R^{-\T} S_B'^{-1/2}
  =
  \nabla_A f \, S_B^{-1/2} Q_B^\T.
\end{equation}
Because $\msign(\cdot)$ is right-equivariant under orthogonal factors,
\begin{equation}
  \msign\!\big(\nabla_A' f \, S_B'^{-1/2}\big)
  =
  \msign\!\big(\nabla_A f \, S_B^{-1/2}\big) Q_B^\T.
\end{equation}
Multiplying again by $S_B'^{-1/2}$ and using \eqref{eq:qb-identities},
\begin{align*}
  \Delta A'
  &=
  -\frac{\eta}{2}
  \msign\!\big(\nabla_A' f \, S_B'^{-1/2}\big) S_B'^{-1/2} \\
  &=
  -\frac{\eta}{2}
  \msign\!\big(\nabla_A f \, S_B^{-1/2}\big) Q_B^\T S_B'^{-1/2}.
\end{align*}
Since $Q_B$ is orthogonal, taking transposes in \eqref{eq:qb-identities} gives
$Q_B^\T S_B'^{-1/2} = S_B^{-1/2} R$, hence
\begin{equation}
  \Delta A'
  =
  -\frac{\eta}{2}
  \msign\!\big(\nabla_A f \, S_B^{-1/2}\big) S_B^{-1/2} R
  =
  \Delta A \, R.
\end{equation}
\end{proof}

\subsection{Gauge invariance of the split weight-decay update}\label{subsec:appendix-splitwd}
\begin{proposition}[Gauge invariance of the split weight-decay update]
Suppose $A' = AR$, $B' = BR^{-\T}$, $\Delta A' = \Delta A R$, and
$\Delta B' = \Delta B R^{-\T}$ for some $R \in \GL(r)$. Define the split weight-decay update by
\begin{equation}
  A_{t+1} = sA_t + s^{-1}\Delta A_t,
  \qquad
  B_{t+1} = sB_t + s^{-1}\Delta B_t,
  \qquad
  s = \sqrt{1-\lambda\eta}.
\end{equation}
Then the induced product update satisfies
$A'_{t+1} B_{t+1}^{\prime\T} = A_{t+1} B_{t+1}^\T$, so the split update is gauge-invariant.
\end{proposition}

\begin{proof}
Applying \eqref{eq:split-wd} to the transformed factors gives
\begin{equation}
  A_{t+1}' = sAR + s^{-1}\Delta A R = (sA + s^{-1}\Delta A)R = A_{t+1}R,
\end{equation}
and similarly $B_{t+1}' = B_{t+1}R^{-\T}$. Therefore
$A_{t+1}' B_{t+1}'^\T = A_{t+1} B_{t+1}^\T$. Equivalently, expanding the product gives
\begin{equation}
  W_{t+1}
  =
  s^2 W_t + A_t \Delta B_t^\T + \Delta A_t B_t^\T + s^{-2}\Delta A_t \Delta B_t^\T,
\end{equation}
and each term is individually gauge-invariant once the factor increments transform equivariantly.
\end{proof}

\subsection{Spectron as spectral renormalization in factor coordinates}\label{subsec:appendix-spectron-factor-coordinates}
Spectron \cite{janson2026spectron} chooses its trust radius in factor coordinates, so gauge sensitivity
is built in from the start. Concretely, it keeps the two-factor orthogonalized update template but
chooses a single shared radius $\rho$ from the factor norms themselves:
\begin{equation}
  \rho^2 + \big(\norm{A}_{2 \to 2} + \norm{B}_{2 \to 2}\big)\rho - \eta = 0.
\end{equation}
This is a reasonable heuristic if one wants to enforce a bound on the composite update norm, but it
lives in factor coordinates rather than quotient geometry. That is why Spectron inherits the
gauge-sensitivity proved in \secref{sec:gauge}: the quantity
$\norm{A}_{2 \to 2} + \norm{B}_{2 \to 2}$ is not invariant under $(A, B) \mapsto (AR, BR^{-\T})$.
\appendixref{subsec:appendix-spectron-alg} gives the reference Spectron update used in this
comparison.

This distinction would likely matter most when finetuning starts from badly imbalanced factors.
Adapter-style finetuning can begin from factorizations whose scales are badly imbalanced, so a
gauge-sensitive trust radius responds to an arbitrary initialization artifact rather than only to
the induced update in $W$.

There is a second, Spectron-specific issue beyond gauge symmetry. LoRA-Muon and LoRA-RITE are
steepest-descent updates driven directly by the current gradient, whereas Spectron first maps
$(\eta, A, B)$ to a radius $\rho(\eta, A, B)$. As a result, its step-size selection rule remains tied to
current factor norms even near stationarity.

\secref{subsec:split-weight-decay} already showed that naive factorwise weight decay induces the
generic product-space mismatch $(1 - \eta \lambda)^2$. Spectron sits on top of that same generic issue,
but the width, depth, and gauge-scale sweeps suggest an additional Spectron-specific effect: its
nonlinear map $\eta \mapsto \rho(\eta, A, B)$ changes the realized update scale with the current factor
norms. We treat that mechanism as the leading interpretation of why Spectron often shifts from the
dense / LoRA-Muon optimum $0.1$ toward $0.1778$.


\subsection{Spectron is not gauge-invariant}\label{subsec:appendix-spectron}
\begin{proposition}[Spectron is not gauge-invariant]\label{prop:spectron-not-gauge}
In the scalar rank-one setting $A = a > 0$, $B = b > 0$, with positive orthogonalized gradient
directions, Spectron chooses a radius $\rho$ from
\begin{equation}
  \rho^2 + (\norm{A}_{2 \to 2} + \norm{B}_{2 \to 2})\rho - \eta = 0.
\end{equation}
Under scalar gauge rescaling $(A,B) \mapsto (cA, B/c)$, both the solution $\rho$ and the induced
first-order product update change with $c$. Therefore Spectron already fails scalar gauge
invariance, and hence cannot be fully gauge-invariant.
\end{proposition}

\begin{proof}
Take the rank-one positive-scalar case $A = a > 0$, $B = b > 0$, and suppose the orthogonalized
gradient directions for both factors equal $1$. Spectron chooses a radius $\rho$ by solving
\begin{equation}
  \rho^2 + (\norm{A}_{2 \to 2} + \norm{B}_{2 \to 2})\rho - \eta = 0.
\end{equation}
Under scalar gauge $(A, B) \mapsto (cA, B/c)$, this becomes
\begin{equation}
  \rho_c^2 + (ca + b/c)\rho_c - \eta = 0.
\end{equation}
Unless $ca + b/c = a + b$, which fails generically for $c \neq 1$, the radius changes. But the
induced first-order update is proportional to
\begin{equation}
  -\rho_c \left(ca + \frac{b}{c}\right),
\end{equation}
which is therefore also gauge-dependent. So Spectron fails scalar gauge invariance, and hence cannot be fully gauge-invariant.
\end{proof}

\subsection{LoRA-RITE as the same spectral update in QR coordinates}\label{subsec:appendix-lorarite-qr}\label{subsec:appendix-rite-equiv}
LoRA-RITE's simplified QR-coordinate core and LoRA-Muon apply the same spectral steepest-descent
update in two different coordinate systems. LoRA-RITE writes that update in QR coordinates rather than in projector or
Gram-matrix form. Let
$B = Q_B R_B$ and $A = Q_A R_A$ be thin QR factorizations with invertible upper-triangular
$R_A, R_B \in \R^{r \times r}$. By unitary invariance of the spectral norm, the two decoupled
subproblems are equivalent to
\begin{align}
  \Delta A^\star R_B^\T
  &=
  \argmin_{Z \in \R^{m \times r}}
  \ip{\nabla_A f \, R_B^{-1}}{Z}
  \quad
  \text{s.t.}
  \quad
  \norm{Z}_{2 \to 2} \le \frac{\eta}{2},
  \label{eq:qr-a}
  \\
  \Delta B^\star R_A^\T
  &=
  \argmin_{Z \in \R^{n \times r}}
  \ip{\nabla_B f \, R_A^{-1}}{Z}
  \quad
  \text{s.t.}
  \quad
  \norm{Z}_{2 \to 2} \le \frac{\eta}{2}.
  \label{eq:qr-b}
\end{align}

Specializing again to the spectral norm gives the QR-coordinate update
\begin{equation}
  \begin{aligned}
    \Delta A^\star
    &=
    -\frac{\eta}{2}\msign(\nabla_A f \, R_B^{-1}) R_B^{-\T},
    \\
    \Delta B^\star
    &=
    -\frac{\eta}{2}\msign(\nabla_B f \, R_A^{-1}) R_A^{-\T}.
  \end{aligned}
  \label{eq:rite-core}
\end{equation}

\begin{proposition}[LoRA-Muon and the simplified LoRA-RITE core coincide under spectral steepest descent]
The QR-coordinate update in \eqref{eq:rite-core} is algebraically identical to the Gram-matrix
LoRA-Muon factor update
\begin{equation}
  \Delta A^\star
  =
  -\frac{\eta}{2}\msign(\nabla_A f\, S_B^{-1/2}) S_B^{-1/2},
  \qquad
  S_B = B^\T B,
\end{equation}
\begin{equation}
  \Delta B^\star
  =
  -\frac{\eta}{2}\msign(\nabla_B f\, S_A^{-1/2}) S_A^{-1/2},
  \qquad
  S_A = A^\T A.
\end{equation}
\end{proposition}

\begin{proof}
Write the polar decomposition $R_B = U_B S_B^{1/2}$ with $U_B$ orthogonal and
$S_B = B^\T B$. Then
\begin{equation}
  R_B^{-1} = S_B^{-1/2} U_B^\T,
  \qquad
  R_B^{-\T} = U_B S_B^{-1/2}.
\end{equation}
Using the right-equivariance of $\msign$ under orthogonal factors,
\begin{align*}
  \msign(\nabla_A f \, R_B^{-1}) R_B^{-\T}
  &=
  \msign(\nabla_A f \, S_B^{-1/2} U_B^\T) U_B S_B^{-1/2} \\
  &=
  \msign(\nabla_A f \, S_B^{-1/2}) U_B^\T U_B S_B^{-1/2} \\
  &=
  \msign(\nabla_A f \, S_B^{-1/2}) S_B^{-1/2}.
\end{align*}
This is exactly \eqref{eq:delta-a}. The $B$ update is identical after swapping $A$ and $B$.
\end{proof}

Thus, at the level of this simplified spectral core, LoRA-Muon and LoRA-RITE are the same update
written in two different coordinate systems.

The practical difference lies in numerical realization and optimizer state. \appendixref{subsec:appendix-rite-algs} gives both
the simplified QR-coordinate realization and the original page-6 LoRA-RITE optimizer. LoRA-Muon
stays QR-free and GEMM-heavy, whereas a QR-coordinate realization requires repeated QR
factorizations that typically want float32-stable arithmetic. The original page-6 LoRA-RITE
algorithm also transports second moments and escaped-mass corrections, whereas LoRA-Muon and the
simplified QR-coordinate core keep only first moments. \appendixref{sec:appendix-costs} gives the exact FLOP and persistent-state
accounting for all four variants.

\section{Numerical Realization}\label{sec:appendix-numerics}

\subsection{Scalar gauge rebalancing}\label{subsec:appendix-gauge-rebalancing}
\begin{corollary}[Gauge rebalancing]
Let an ideal LoRA-Muon step be defined through the gauge-invariant $\Delta W$ above. Then any
deterministic reparameterization $(A, B) \mapsto (AR, BR^{-\T})$ with $R \in \GL(r)$ may be
inserted between steps without changing the induced update on $W = AB^\T$, provided any first
moment variables are transported by the same gauge action as the corresponding factor gradients.
\end{corollary}

\begin{algorithm}[t]
\caption{(Periodic) Scalar Gauge Rebalancing}
\label{alg:gauge-rebalance}
\small
\begin{algorithmic}[1]
\REQUIRE Factors $A \in \R^{m \times r}$, $B \in \R^{n \times r}$, first moments $m^A \in \R^{m \times r}$, $m^B \in \R^{n \times r}$, damping exponent $\alpha \in (0,1]$
\STATE $c \gets \left(\texttt{power\_iterate}(B) / \texttt{power\_iterate}(A)\right)^{\alpha/2}$
\STATE $A \gets cA$ and $B \gets B/c$
\STATE $m^A \gets m^A/c$ and $m^B \gets c m^B$
\STATE \textbf{return} $A, B, m^A, m^B$
\end{algorithmic}
\end{algorithm}
The update on $W$ is gauge-invariant, but the choice of representative $(A,B)$ inside a gauge class
is not. We use scalar gauge rebalancing to choose a numerically better representative without
changing the underlying weight-space step. Algorithm~\ref{alg:gauge-rebalance} therefore acts as a
conditioning tool: it reduces factor imbalance while leaving the update on $W$ untouched. The explicit
$\texttt{power\_iterate}(\cdot)$ call only realizes the spectral balancing heuristic numerically; it
does not change the theory. \appendixref{subsec:appendix-moments} proves the moment-transport rule
that makes this scalar reparameterization consistent.

\subsection{Moment transport under scalar gauge rebalancing}\label{subsec:appendix-moments}
\begin{proposition}[Why the moments transform oppositely]
Let $W = AB^\T$, let the factor gradients be $g^A = \nabla_A f = G B$ and
$g^B = \nabla_B f = G^\T A$, and let the EMA states update as
\begin{equation}
  m_{t+1}^A = \beta m_t^A + (1-\beta) g_t^A,
  \qquad
  m_{t+1}^B = \beta m_t^B + (1-\beta) g_t^B.
\end{equation}
Under the scalar gauge action $A' = cA$, $B' = B/c$, the factor gradients satisfy
\begin{equation}
  g^{A\prime} = g^A / c,
  \qquad
  g^{B\prime} = c g^B.
\end{equation}
Therefore the EMA state remains equivariant only if it transforms as
\begin{equation}
  m^{A\prime} = m^A / c,
  \qquad
  m^{B\prime} = c m^B.
\end{equation}
\end{proposition}

\begin{proof}
Because $A'B'^\T = AB^\T$, the ambient gradient $G = \nabla_W f(W_{\mathrm{pre}} + AB^\T)$ is
unchanged by the scalar gauge transformation. By chain rule,
\begin{equation}
  g^A = \nabla_A f = G B,
  \qquad
  g^B = \nabla_B f = G^\T A.
\end{equation}
Under $A' = cA$ and $B' = B/c$, this becomes
\begin{equation}
  g^{A\prime} = G(B/c) = g^A/c,
  \qquad
  g^{B\prime} = G^\T(cA) = c g^B.
\end{equation}
Now consider the EMA updates
\begin{equation}
  m_{t+1}^A = \beta m_t^A + (1-\beta) g_t^A,
  \qquad
  m_{t+1}^B = \beta m_t^B + (1-\beta) g_t^B.
\end{equation}
If we transport the moments as $m_t^{A\prime} = m_t^A/c$ and $m_t^{B\prime} = c m_t^B$, then
\begin{equation}
  m_{t+1}^{A\prime}
  =
  \beta m_t^A/c + (1-\beta) g_t^A/c
  =
  m_{t+1}^A/c,
\end{equation}
and similarly $m_{t+1}^{B\prime} = c m_{t+1}^B$. Thus the EMA recursion commutes with scalar
gauge rebalancing only under the opposite scaling rule shown in
\hyperref[alg:gauge-rebalance]{Algorithm~\ref*{alg:gauge-rebalance}}.
\end{proof}

\subsection{Muon Newton--Schulz orthogonalization}\label{subsec:appendix-orthog}
We realize the matrix sign operator $\msign$ numerically with a Muon-style Newton--Schulz
orthogonalization pass. Given an input matrix $M$, we optionally transpose to the smaller leading
dimension, normalize by the Frobenius norm, and then iterate a low-degree polynomial map that
drives the singular values toward $1$. The same primitive powers ambient Muon, and here it realizes
$\msign$ for LoRA-Muon as well.

\begin{algorithm}[tbp]
\caption{Muon Newton--Schulz Orthogonalization}
\label{alg:appendix-orthog}
\small
\begin{algorithmic}[1]
\REQUIRE Matrix $M$, coefficient list $\{(a_k,b_k,c_k)\}_{k=0}^{K-1}$, $\epsilon = \texttt{1e-20}$
\STATE $M_0 \gets M$
\STATE $\texttt{transpose} \gets (\text{rows}(M_0) > \text{cols}(M_0))$
\IF{$\texttt{transpose}$}
  \STATE $M_0 \gets M_0^\T$
\ENDIF
\STATE $M_0 \gets M_0 / (\norm{M_0}_F + \epsilon)$
\FOR{$k = 0, \dots, K-1$}
  \STATE $U_k \gets M_k M_k^\T$
  \STATE $M_{k+1} \gets a_k M_k + (b_k U_k + c_k U_k^2)M_k$
\ENDFOR
\IF{$\texttt{transpose}$}
  \STATE $M_K \gets M_K^\T$
\ENDIF
\STATE \textbf{return} $M_K$
\end{algorithmic}
\end{algorithm}

The numerical realization used in our experiments applies the Polar Express matrix-sign
coefficients from \cite{amsel2026polarexpress}, listed in Table~\ref{tab:orthog-coeffs}.

\begin{table}[tbp]
  \centering
  \small
  \begin{tabular}{rccc}
    \toprule
    Step $k$ & $a_k$ & $b_k$ & $c_k$ \\
    \midrule
    0 & 7.2086 & -15.5131 & 9.0178 \\
    1 & 3.9623 & -2.5813 & 0.4542 \\
    2 & 3.9466 & -2.5765 & 0.4544 \\
    3 & 3.8991 & -2.5671 & 0.4566 \\
    4 & 3.7186 & -2.5308 & 0.4653 \\
    5 & 3.1390 & -2.3073 & 0.4733 \\
    6 & 2.1715 & -1.5246 & 0.3885 \\
    7 & 1.8648 & -1.2224 & 0.3577 \\
    \bottomrule
  \end{tabular}
  \caption{Coefficient sequence for Algorithm~\ref{alg:appendix-orthog}, which realizes the
  Muon-style matrix-sign map $\msign(\cdot)$ by a polynomial Newton--Schulz iteration after
  Frobenius normalization. Each row gives the coefficients for one orthogonalization step in the
  numerical realization used in our experiments.}
  \label{tab:orthog-coeffs}
\end{table}

\subsection{Newton--Schulz inverse square roots}\label{subsec:appendix-invroot}
The LoRA-Muon factor update also needs inverse square roots of positive semidefinite matrices. In
the main algorithm, those matrices are the Gram matrices $S_A = A^\T A$ and $S_B = B^\T B$, but
the numerical routine itself only needs a PSD input $P$. We first normalize $P$ by its Frobenius
norm, regularize it with a small diagonal shift, and then apply a polynomial Newton--Schulz
recurrence. Algorithm~\ref{alg:appendix-invroot} uses a single PSD matrix to keep the presentation
simple. In practice, the same kernel applies to a stack of same-shaped PSD matrices; for LoRA-Muon,
that means batching $S_A$ and $S_B$ within each layer, and in the common fixed-rank setting even
batching the full collection of Gram matrices across layers.

\begin{algorithm}[tbp]
\caption{Newton--Schulz Inverse Square Root}
\label{alg:appendix-invroot}
\small
\begin{algorithmic}[1]
\REQUIRE PSD matrix $P$, coefficient list $\{(a_k,b_k,c_k)\}_{k=0}^{K-1}$, $\epsilon = \texttt{1e-5}$, $\gamma = 1.001$
\STATE $t \gets \norm{P}_F$
\STATE $P_0 \gets P/t + \epsilon I$
\STATE $X_0 \gets I$
\FOR{$k = 0, \dots, K-1$}
\STATE $W_k \gets (a_k/\gamma)I + (b_k/\gamma^3)P_k + (c_k/\gamma^5)P_k^2$
\STATE $X_{k+1} \gets X_k W_k$
\STATE $P_{k+1} \gets \sym(P_k W_k^2)$
\ENDFOR
\STATE \textbf{return} $t^{-1/2} X_K$
\end{algorithmic}
\end{algorithm}

For the $r$th-root iteration specialized to $r=2$, the unscaled coefficient list used in our
experiments is given in Table~\ref{tab:invroot-coeffs}.

\begin{table}[tbp]
  \centering
  \small
  \begin{tabular}{rccc}
    \toprule
    Step $k$ & $a_k$ & $b_k$ & $c_k$ \\
    \midrule
    0 & 7.424865680309214 & -18.39581635618996 & 12.896720413604342 \\
    1 & 3.4877256051546017 & -2.3300436563986993 & 0.4404692168431095 \\
    2 & 2.7766085124882527 & -2.070643152532662 & 0.46302261050004967 \\
    3 & 1.9913142104341506 & -1.373936700681269 & 0.3875934979568538 \\
    4 & 1.8754637749479246 & -1.2505152090010534 & 0.37505152463617264 \\
    5 & 1.874999066623701 & -1.2499981332141676 & 0.37499906659046633 \\
    6 & 1.875 & -1.25 & 0.375 \\
    \bottomrule
  \end{tabular}
  \caption{Unscaled coefficient sequence for the $r$th-root Newton--Schulz iteration specialized to
  $r=2$, as used by Algorithm~\ref{alg:appendix-invroot} to compute inverse square roots of PSD
  matrices such as the LoRA Gram matrices $A^\T A$ and $B^\T B$. Each step applies the scaled
  coefficients $a_k/\gamma$, $b_k/\gamma^3$, and $c_k/\gamma^5$ with $\gamma = 1.001$.}
  \label{tab:invroot-coeffs}
\end{table}

\subsection{Reference Spectron update}\label{subsec:appendix-spectron-alg}
Algorithm~\ref{alg:appendix-spectron} gives the Spectron-style optimizer analyzed in \secref{sec:spectron-rite} and
used for comparison in our experiments. It keeps first moments on the individual factors, orthogonalizes those
moments, and chooses a shared radius from the current factor norms.

\begin{algorithm}[tbp]
\caption{Spectron}
\label{alg:appendix-spectron}
\small
\begin{algorithmic}[1]
\REQUIRE Factors $A_t, B_t$, previous first moments $m_{t-1}^A, m_{t-1}^B$, learning rate $\eta_t$, momentum $\beta$, weight decay $\lambda$
\STATE $g_t^A \gets \nabla_A f(W_{\mathrm{pre}} + A_t B_t^\T)$ and $g_t^B \gets \nabla_B f(W_{\mathrm{pre}} + A_t B_t^\T)$
\STATE $m_t^A \gets \beta m_{t-1}^A + (1-\beta) g_t^A$ and $m_t^B \gets \beta m_{t-1}^B + (1-\beta) g_t^B$
\STATE $u_t^A \gets \msign(m_t^A)$ and $u_t^B \gets \msign(m_t^B)$
\STATE $s_t \gets \texttt{power\_iterate}(A_t) + \texttt{power\_iterate}(B_t)$
\STATE $\rho_t \gets \frac{1}{2}\left(-s_t + \sqrt{s_t^2 + 4\eta_t}\right)$
\STATE $A_{t+1} \gets (1-\lambda\eta_t)A_t - \rho_t u_t^A$
\STATE $B_{t+1} \gets (1-\lambda\eta_t)B_t - \rho_t u_t^B$
\STATE \textbf{return} $A_{t+1}, B_{t+1}, m_t^A, m_t^B$
\end{algorithmic}
\end{algorithm}

\subsection{Reference LoRA-RITE updates}\label{subsec:appendix-rite-algs}
Algorithm~\ref{alg:appendix-rite-simple} gives the simplified QR-coordinate spectral update that is
algebraically equivalent to LoRA-Muon. Algorithm~\ref{alg:appendix-rite-original} then gives the
original page-6 LoRA-RITE optimizer on the $A$ side, including transported moments and escaped
mass; the $B$ side is symmetric.

\begin{algorithm}[tbp]
\caption{Simplified LoRA-RITE}
\label{alg:appendix-rite-simple}
\small
\begin{algorithmic}[1]
\REQUIRE Factors $A_t, B_t$, previous first moments $m_{t-1}^A, m_{t-1}^B$, learning rate $\eta_t$, momentum $\beta$
\STATE $g_t^A \gets \nabla_A f(W_{\mathrm{pre}} + A_t B_t^\T)$ and $g_t^B \gets \nabla_B f(W_{\mathrm{pre}} + A_t B_t^\T)$
\STATE $m_t^A \gets \beta m_{t-1}^A + (1-\beta) g_t^A$ and $m_t^B \gets \beta m_{t-1}^B + (1-\beta) g_t^B$
\STATE $(Q_A, R_A) \gets \texttt{thin\_qr}(A_t)$ and $(Q_B, R_B) \gets \texttt{thin\_qr}(B_t)$
\STATE $Y_A \gets \texttt{solve\_right\_triangular}(m_t^A, R_B)$
\STATE $Y_B \gets \texttt{solve\_right\_triangular}(m_t^B, R_A)$
\STATE $Z_A \gets \msign(Y_A)$ and $Z_B \gets \msign(Y_B)$
\STATE $\Delta A_t \gets -\frac{\eta_t}{2}\,\texttt{solve\_right\_triangular}(Z_A, R_B^\T)$
\STATE $\Delta B_t \gets -\frac{\eta_t}{2}\,\texttt{solve\_right\_triangular}(Z_B, R_A^\T)$
\STATE $A_{t+1} \gets A_t + \Delta A_t$ and $B_{t+1} \gets B_t + \Delta B_t$
\STATE \textbf{return} $A_{t+1}, B_{t+1}, m_t^A, m_t^B$
\end{algorithmic}
\end{algorithm}

\begin{algorithm}[tbp]
\caption{Original LoRA-RITE Algorithm 1 (A-side only)}
\label{alg:appendix-rite-original}
\small
\begin{algorithmic}[1]
\REQUIRE Factor $A_t$, companion factor $B_t$, gradient $\nabla A_t$, previous basis $U_{B,t-1}$, unmagnified moments $\overline{M}_{A,t-1}, \overline{V}_{A,t-1}$, escaped mass $\rho_{A,t-1}$, learning rate $\eta_t$, momentum $\beta_1$
\STATE Compute the polar decomposition $B_t = U_{B,t} R_{B,t}$
\IF{$t=1$}
  \STATE $P_{A,t} \gets I$
\ELSE
  \STATE $P_{A,t} \gets U_{B,t}^\T U_{B,t-1}$
\ENDIF
\STATE $\overline{\nabla}_{A,t} \gets \texttt{solve\_right\_triangular}(\nabla A_t, R_{B,t})$
\STATE $\overline{V}_{A,t} \gets P_{A,t}\overline{V}_{A,t-1}P_{A,t}^\T + \overline{\nabla}_{A,t}^\T \overline{\nabla}_{A,t} / m$
\STATE $\rho_{A,t} \gets \rho_{A,t-1} + d_\lambda\!\left(\overline{V}_{A,t-1}, P_{A,t}\overline{V}_{A,t-1}P_{A,t}^\T\right)$
\STATE $\overline{S}_{A,t} \gets \overline{\nabla}_{A,t}\left(\overline{V}_{A,t} + \rho_{A,t} I\right)^{-1/2}$
\STATE $\overline{M}_{A,t} \gets \beta_1 \overline{M}_{A,t-1}P_{A,t}^\T + (1-\beta_1)\overline{S}_{A,t}$
\STATE $A_{t+1} \gets A_t - \eta_t\,\texttt{solve\_right\_triangular}(\overline{M}_{A,t}, R_{B,t}^\T)$
\STATE \textbf{return} $A_{t+1}, U_{B,t}, \overline{M}_{A,t}, \overline{V}_{A,t}, \rho_{A,t}$
\STATE \textit{The $B$-side update is symmetric and is omitted only to match the original presentation.}
\end{algorithmic}
\end{algorithm}

\section{Optimizer Cost and State Accounting}\label{sec:appendix-costs}

Throughout this appendix, let
\begin{equation}
  A \in \R^{m \times k},
  \qquad
  B \in \R^{n \times k},
  \qquad
  k \ll m \le n.
\end{equation}
We write $T_o$ for the number of Newton--Schulz orthogonalization steps, $T_r$ for the number of
inverse-root Newton--Schulz steps, and $T_p$ for the number of power-iteration steps used for
Spectron's spectral-norm estimate. The dense Muon row below is a full $m \times n$ matrix
baseline; the remaining rows are low-rank factor-pair costs.

\subsection{Per-Pair FLOP Cost and Persistent Auxiliary State}\label{subsec:appendix-cost-flops}
\begin{table}[tbp]
  \centering
  \resizebox{0.98\textwidth}{!}{%
  \begin{tabular}{l>{\centering\arraybackslash}m{0.47\textwidth}>{\centering\arraybackslash}m{0.22\textwidth}}
    \toprule
    Method & Optimizer-step FLOPs & Persistent optimizer state \\
    \midrule
    Dense Muon (full $m \times n$ matrix) &
    $T_o(4nm^2 + 2m^3)$ &
    $mn$ \\
    LoRA-Muon &
    $(6 + 4T_o)(m+n)k^2 + (16T_r + 4T_o)k^3$ &
    $(m+n)k$ \\
    Spectron &
    $4T_o(m+n)k^2 + 4T_ok^3 + 4T_p(m+n)k$ &
    $(m+n)k$ \\
    Simplified LoRA-RITE &
    $(4 + 4T_o)(m+n)k^2 + \left(4T_o - \frac{4}{3}\right)k^3$ &
    $(m+n)k$ \\
    LoRA-RITE (full page-6 Algorithm 1) &
    $(10 + 2T_r)(m+n)k^2 + \left(\frac{128}{3} + 12T_r\right)k^3$ &
    $2(m+n)k + 2k^2 + 2$ \\
    \bottomrule
  \end{tabular}%
  }
  \caption{Per-update optimizer cost under the counting model in \appendixref{sec:appendix-costs},
  assuming $A \in \R^{m \times k}$, $B \in \R^{n \times k}$, and $m \le n$. The dense Muon row is a
  full $m \times n$ matrix baseline, while the remaining rows report the cost of updating one
  low-rank factor pair. The simplified LoRA-RITE row is the first-moment-only QR-coordinate
  spectral core, while the final row is the full page-6 optimizer with transported second moments
  and escaped mass. The FLOP column counts only the optimizer update itself, and the state column
  counts persistent auxiliary tensors rather than model parameters.}
  \label{tab:cost-flops}
\end{table}
Neither FLOPs nor persistent auxiliary state, by themselves, capture the kernel-level advantages of
QR-free, GEMM-heavy realizations or the precision sensitivity of QR-based alternatives. In
particular, simplified LoRA-RITE shares first-moment-only state with LoRA-Muon and Spectron, but
it still realizes the spectral step through QR and triangular solves rather than through Gram
matrices and inverse roots.

\section{Experimental Details}\label{sec:appendix-experiments}
The main experimental configuration is a TinyShakespeare character-level transformer with $d_{\mathrm{model}}
= 128$, $2$ layers, $2$ heads, batch size $64$, sequence length $128$, and a budget of $1.048576$ M
tokens per run. Token embeddings use AdamW \cite{kingma2015adam,loshchilov2019adamw}, the LM head uses dense Muon, and LoRA is applied
only to the backbone blocks. The rank sweep runs ranks $1, 2, 4, 8, 16,$ and $32$ across the learning-rate grid
\[
  \{0.01, 0.0178, 0.0316, 0.0562, 0.1, 0.1778, 0.3162, 0.5623, 1.0\}
\]
with six seeds. The width sweep runs widths $64$ and $128$, the depth sweep runs
depths $1, 2,$ and $4$ at fixed LoRA rank $32$, and the gauge-scale sweep runs scales $1, 3, 9,$
and $27$.

\subsection{Compute-match accounting}\label{subsec:appendix-flops}
The rank-transfer sweep also tracks how many LoRA-Muon optimizer steps fit inside the compute
budget of the dense reference run, which consumes $1.048576$ M tokens. The resulting multipliers are:

\begin{table}[tbp]
  \centering
  \begin{tabular}{rccc}
    \toprule
    Rank & Dense steps & LoRA-Muon steps & Multiplier \\
    \midrule
    1 & 128 & 1297 & $10.1311\times$ \\
    2 & 128 & 1181 & $9.2253\times$ \\
    4 & 128 & 1002 & $7.8230\times$ \\
    8 & 128 & 767 & $5.9917\times$ \\
    16 & 128 & 521 & $4.0655\times$ \\
    32 & 128 & 314 & $2.4464\times$ \\
    \bottomrule
  \end{tabular}
  \caption{Compute-match step multipliers in the TinyShakespeare rank-transfer sweep. For each LoRA
  rank, the dense reference budget is fixed at $128$ training steps, and the LoRA-Muon column
  reports how many low-rank steps fit under the same total compute budget. Smaller ranks therefore
  permit many more optimizer steps at fixed compute, with rank $4$ allowing roughly $7.8\times$ as
  many steps as the matched dense run.}
  \label{tab:compute-match}
\end{table}

At rank $2$, the proxy already matches the dense best tested learning rate while using only about
$4.3\%$ as many trainable parameters as the matched dense model and supporting $1181$
compute-matched optimizer steps. By rank $4$, the step budget rises to $1002$ steps while the
remaining dense-matrix FLOPs are already dominated by the LM-head matmul, which accounts for
roughly $70\%$ of the residual dense-matrix cost. This is the floor that eventually limits further
speedup as the low-rank backbone gets cheaper.

\subsection{Gauge probes and paired studies}\label{subsec:appendix-studies}
The gauge-probe numbers in \secref{sec:experiments} come from dedicated one-step invariance
evaluations, and the gauge-rebalancing statistics come from paired multi-seed comparisons.



\end{document}